\documentclass[ijds,nonblindrev]{informs4}

\OneAndAHalfSpacedXI

\newcommand{\R}{\mathbb{R}} 
\renewcommand{\P}[1]{\mathbb{P}\left(#1 \right)} 
\newcommand{\E}[1]{\mathbb{E}\left[ #1 \right]} 
\newcommand{\N}{\mathbb{N}} 
\renewcommand{\S}{\mathbb{S}} 

\definecolor{warningcol}{rgb}{.99,.1,.5}
\definecolor{todocol}{rgb}{.4,.4,.8}
\definecolor{sketchcol}{rgb}{.4,.4,.8}
\definecolor{outlinecol}{rgb}{.8,.4,.3}

\renewcommand{\E}{\mathbb{E}}
\newcommand{\Q}{\mathbb{Q}}
\newcommand{\eps}{\varepsilon}

\newcommand{\Qr}[1]{\mathbb{Q}\pbr{#1}}

\newcommand{\pbr}[1]{\left( #1\right)}
\newcommand{\sbr}[1]{\left[ #1\right]}
\newcommand{\cbr}[1]{\left\{ #1\right\}}

\newcommand{\aln}[1]{\begin{align} #1 \end{align}}
\newcommand{\alns}[1]{\begin{align*} #1 \end{align*}}

\makeatother

\renewcommand{\S}{\mathcal{S}}
\newcommand{\A}{\mathcal{A}}
\renewcommand{\P}{\mathbb{P}}
\renewcommand{\Q}{\mathbb{Q}}

\newcommand{\Qt}{\widetilde{\Q}}

\newcommand{\T}{\mathsf{T}}

\newcommand{\Lf}{\Lambda_f}

\usepackage{algorithm}
\usepackage{algpseudocode}
\usepackage{mathtools}
\usepackage[normalem]{ulem}

\newcommand\footnoteref[1]{\protected@xdef\@thefnmark{\ref{#1}}\@footnotemark}

\usepackage{amsmath,amssymb,amsfonts}

\usepackage{natbib}
 \bibpunct[, ]{(}{)}{,}{a}{}{,}%
 \def\bibfont{\small}%
\usepackage{hyperref}
\usepackage{rotating}
\usepackage{fancyvrb}
\usepackage{comment}
\TheoremsNumberedThrough     
\ECRepeatTheorems
\JOURNAL{Operations Research}

\EquationsNumberedThrough    

\MANUSCRIPTNO{IJDS-0001-1922.65}
\usepackage{mathrsfs}
\usepackage[skip=1pt,font=large]{caption}
\usepackage{subcaption}
\newtheorem{innercustomthm}{Theorem}
\newenvironment{customthm}[1]
  {\renewcommand\theinnercustomthm{#1}\innercustomthm}
  {\endinnercustomthm}

\begin{document}



\TITLE{On the Convergence of Modified Policy Iteration in Risk Sensitive Exponential Cost Markov Decision Processes}

\ARTICLEAUTHORS{%
\AUTHOR{Yashaswini Murthy}

\AFF{Electrical and Computer Engineering and Coordinated Science Laboratory, University of Illinois Urbana-Champaign, \EMAIL{ymurthy2@illinois.edu}}

\AUTHOR{Mehrdad Moharrami}

\AFF{Computer Science, University of Iowa, \EMAIL{moharami@uiowa.edu}}

\AUTHOR{R. Srikant}

\AFF{Electrical and Computer Engineering and Coordinated Science Laboratory, University of Illinois Urbana-Champaign, \EMAIL{rsrikant@illinois.edu}}
}

\ABSTRACT{%
Modified policy iteration (MPI) is a dynamic programming algorithm that combines elements of policy iteration and value iteration. The convergence of MPI has been well studied in the context of discounted and average-cost MDPs. In this work, we consider the exponential cost risk-sensitive MDP formulation, which is known to provide some robustness to model parameters. Although policy iteration and value iteration have been well studied in the context of risk sensitive MDPs, MPI is unexplored. We provide the first proof that MPI also converges for the risk-sensitive problem in the case of finite state and action spaces. Since the exponential cost formulation deals with the multiplicative Bellman equation, our main contribution is a convergence proof which is quite different than existing results for discounted and risk-neutral average-cost problems as well as risk sensitive value and policy iteration approaches. We conclude our analysis with simulation results, assessing MPI's performance relative to alternative dynamic programming methods like value iteration and policy iteration across diverse problem parameters. Our findings highlight risk-sensitive MPI's enhanced computational efficiency compared to both value and policy iteration techniques. }

\maketitle

\section{Introduction}\label{sec1}

We consider stochastic control problems over finite state and action spaces, also known as Markov Decision Processes (MDPs). Various applications, including communication systems, autonomous car navigation, and queuing systems, are expressed as Markov Decision Processes (MDPs) for the purpose of control. Conventional methods for determining optimal policies in these MDPs typically employ dynamic programming techniques such as policy iteration, value iteration, or linear programming (\cite{bertsekasvol1}, \cite{bertsekasvol2}, \cite{puterman}). 

Reinforcement learning attempts to solve the control problem when the probability transition matrix is either unknown or the probability transition matrix is known but the state space is very large to obtain exact solutions (\cite{sutton}). A considerable portion of earlier research in this domain is centered on discounted-cost problems or average-cost Markov Decision Processes (MDPs), without accounting for potential distributional differences in the underlying probability transition kernel that characterizes the environment in which the agent operates. In this paper, our focus is on the robust version of the average-cost problem. While this paper employs the KL divergence metric to model distributional uncertainties, existing literature captures robustness through various other metrics, including the Wasserstein distance (\cite{yang2017convex,chen2019distributionally,clement2021first} and state-dependent uncertainty sets(\cite{iyengar2005robust,mannor2016robust}).

Robust control problems with linear state-space and quadratic costs have been well studied in the control theory literature (\cite{doylezhou}, \cite{dullerud}, \cite{basar}). It is also well-known that these robust control problems are closely related to the control of systems with a risk-sensitive exponential cost (\cite{whittle}). Here, we consider the finite-state, finite-action counterpart of such robust/risk-sensitive control problems (\cite{borkar2,borkar3,borkar4}). Unlike, the LQG setting in \cite{whittle}, the risk-sensitive MDP does not admit a closed-form solution even when the system model is known. 

Within the domain of risk-sensitive Markov Decision Processes (MDPs), the focus in dynamic programming has largely revolved around value iteration and policy iteration. The reinforcement learning counterparts of these algorithms have also undergone thorough investigation. This paper presents an alternative dynamic programming algorithm, termed modified policy iteration, which incorporates favorable characteristics from both value iteration and policy iteration algorithms. Value iteration involves iteratively applying the optimal Bellman operator, whereas policy iteration involves successive iterations of policy improvement followed by perfect policy evaluation. In modified policy iteration, rather than relying on perfect policy evaluation, we approximately evaluate the value function by repeatedly applying the Bellman operator corresponding to the improved policy a finite number of times.

The approximate nature of modified policy iteration can be used to obtain finite-time performance bounds for policy-based reinforcement learning, where each iteration of the RL algorithm can be viewed as an approximate policy evaluation and an approximate policy improvement (\cite{winnicki2021,winnicki2023convergence}). However, such results have been obtained only for the classical, non-robust versions of MDPs. While there has been work on the asymptotic theory of RL algorithms for risk-sensitive MDPs (\cite{borkar2,borkar3,mohar}), there has not been much work on the finite-time performance of such algorithms except in the case of value-based algorithms (\cite{fei}). While we do not study RL algorithms in this paper, we hope that the results in this paper can be used as a first step to analyze RL algorithms for robust MDPs. 

\subsubsection*{Main Contributions:}
\begin{itemize}
    \item We introduce a new dynamic programming algorithm, referred to as modified policy iteration, within the framework of risk sensitive exponential cost MDPs, which leverages properties of both value and policy iteration. 
    \item We analyze the modified policy iteration algorithm and provide asymptotic as well as finite-time convergence bounds for the sequence of policies derived from this algorithm towards the globally optimal robust policy. 
    \item The presence of partial policy evaluation means that the properties facilitating the convergence of policy and value iteration no longer apply in the context of this algorithm. Consequently, we devise novel proof techniques in order to establish the convergence analysis mentioned above. Additionally, we introduce a form of normalization within the algorithm and demonstrate its crucial role in ensuring the boundedness of consecutive value function iterates.
    \item We include simulation results that assess the performance of modified policy iteration relative to policy and value iteration. The comparison encompasses convergence performance across various parameters, including the risk sensitivity parameter and the extent of partial policy evaluation.
\end{itemize}

\subsubsection*{Paper Outline:}
\begin{itemize}
    \item Section \ref{priorwork}: We discuss prior work on risk sensitive MDPs and modified policy iteration in risk neutral MDPs and their reinforcement learning applications.
    \item Section \ref{prelim}: We introduce the essential preliminaries, state our model and assumptions, and present the algorithm studied in this paper.
    \item Section \ref{results}: We present our convergence result for risk-sensitive modified policy iteration and outline its proof.
    \item Section \ref{proofs}: We present the proof of convergence, including all the supporting lemmas used in the convergence analysis of risk sensitive modified policy iteration. 
    \item Section \ref{simulations}: We present simulation results in this section. 
    \item Section \ref{conclusion}: Conclusions and future works are discussed in this section.
\end{itemize}
A preliminary version of this work without proofs and simulations appeared at Learning for Dynamics and Control (L4DC 2023) conference (\cite{murthy2023modified}).

\section{Prior Work}\label{priorwork}
In this section we discuss previous work in the domain of robust MDPs and the use of modified policy iteration in risk neutral MDPs. 

\subsection{Risk Sensitive MDPs}
The focus of dynamic programming techniques in the context of risk-sensitive Markov Decision Processes (MDPs) has predominantly revolved around policy and value iteration. In \cite{borkar1}, the authors provide convergence analysis for policy and value iteration in the context of countably infinite state space and finite action space. The asymptotic bounds leverage the fact that consecutive iterations of these algorithms yield monotonically decreasing risk sensitive average costs. \cite{biel} and \cite{cadena} consider finite state-action space value iteration. The analysis in \cite{cadena} depends crucially on value function iterates being related through the optimal risk sensitive Bellman Operator. They provide asymptotic convergence bounds under relaxed assumptions, i.e., they assume a solution to the optimal multiplicative Bellman operator exists but do not explicitly rely on irreducibility of the associated Markov chains. The monotonicity associated with the optimal Bellman operator, in addition to irreducibility and aperiodicity of the Markov chain is leveraged in \cite{biel} to provide finite time bounds for convergence of value iteration. This is achieved through a span seminorm contraction analysis with the Birkhoff coefficient serving as the contraction factor. However, neither of these analyses can be seamlessly extended to address the convergence of modified policy iteration. This limitation arises from the introduction of partial policy evaluation, leading to the breakdown of monotonicity arguments central to the approaches presented in \cite{cadena} and \cite{biel}. In recent literature, \cite{fei} consider risk sensitive value iteration in the context of finite horizon MDPs. They focus on optimizing the total reward over episodes of finite length and do not take into account the objective of the infinite horizon risk-sensitive average reward, which is the primary goal of this paper. \cite{gvm} utilizes time dependent risk factors and derives dynamic programming techniques to solve for the risk sensitive optimal policy when the objective is finite horizon discounted reward. 

Outside the exponential risk formulation, there are several other notions of robustness employed to account for model uncertainties in literature. In the context of discounted MDPs, a common approach is to model the ambiguity in the system parameters through state-dependent rectangular sets, where these sets are assumed to be mutually independent (\cite{iyengar2005robust}). However, this leads to extremely conservative solutions as the adversary can leverage the worst case situation for every state. \cite{goyal2021robust,mannor2016robust,xu2010distributionally} propose alternative modeling techniques which alleviate the aforementioned conservative nature of the solution. Apart from the KL divergence metric considered in the present work, the distributional difference is also quantified through the Wasserstein metric (\cite{yang2017convex,clement2021first}). A broader framework to model robustness is provided in \cite{chen2019distributionally}. 



\subsection{Modified Policy Iteration in risk neutral MDPs}

Recently, several papers have studied RL problems using versions of dynamic programming techniques that are computationally more tractable compared to traditional value iteration or policy iteration (\cite{efroni}, \cite{winnicki2021}). These algorithms use two key ideas: (i) modified policy iteration: some version of policy iteration is used, where instead of exact policy evaluation, a few iterations of fixed-point iterations are performed (\cite{puterman}), and (ii) approximate policy iteration: both the policy evaluation and the few iterations of fixed-point iterations mentioned in (i) are performed approximately (\cite{bertsekasvol2}). As shown in \cite{efroni,winnicki2021}, modified policy iteration and approximate policy iteration can be used to model the concepts used in practical RL algorithms such as tree search, rollout, lookahead, and function approximation. However, all the known results in this context are for the discounted-cost or average-cost risk neutral infinite-horizon problem.

To develop an analog of the rich theory that exists for discounted-cost problems, one has to first develop a theory for modified policy iteration and approximate policy iteration in the context of risk-sensitive exponential cost MDPs. For risk-neutral average cost problems, there exists a theory of modified policy iteration (\cite{vanderwal}).  For risk-sensitive MDPs, we are unaware of any results for either modified policy iteration or approximate policy iteration. In this paper, as a first step towards developing a theory of RL for risk-sensitive problems with known but large probability transition matrices, we define the equivalent of modified policy iteration in the case of risk-sensitive MDPs and prove that it converges. In the case of discounted-cost problems and average-cost problems, the proof of convergence relies on the properties of the Bellman operator which is additive in nature. Our main contribution in this paper is to develop risk-sensitive counterparts of the results in \cite{vanderwal}, which deals with the risk-neutral setting. Specifically, the fact that the Bellman operator has multiplicative terms, instead of additive terms, means that we have to deal with potential unboundeness in the iterates of the modified policy iteration algorithm. We will detail the differences in the proof techniques when we present the mathematical results later in the paper.


\section{Preliminaries}\label{prelim}
In this section, we introduce risk-sensitive average cost MDPs and the associated multiplicative Bellman Operator. We also discuss our assumptions and present the modified policy iteration algorithm that is studied in this paper.

\subsection{Risk-Sensitive Markov Decision Process}

We consider an infinite horizon Markov decision process with finite state space $\S$, finite action space $\A$, and transition kernel $\P$. The class of deterministic policies is denoted by $\Pi = \{ f \colon \S \to \A\}$, where each policy assigns an action to each state. Given a policy $f\in\Pi$, the underlying Markov process is denoted by $\P_f:\S\rightarrow\S$, where $\P_f(s'|s) \coloneqq \P(s'|s,f(s))$ is the probability of moving to state $s'\in\S$ from state $s\in\S$ upon taking action $f(s)\in\A$. Associated with each state-action pair $(s,f(s))$, is a one-step cost which is denoted by $c_f(s) \coloneqq c(s,f(s)) \in [\underline{c},\overline{c}]$. 

To motivate our risk-sensitive cost formulation in the next subsection, we first briefly recall that, in the traditional average cost MDP the average cost $J_f$ associated with a deterministic policy $f \in \Pi$ is given by
\alns{
J_f = \lim_{t\to\infty}\frac{1}{t}\E\sbr{\sum_{k=0}^{t-1}c_f(s_k)}.
}
Here $s_k$ represents the state at time $k$ and the expectation is taken with respect to the transition probability $\P_f$ associated with the policy $f$. The limit mentioned above is well-defined and is independent of the initial state distribution, subject to certain mild conditions on the underlying Markov decision process (\cite{puterman,bertsekasvol2}). 

Let the stationary distribution over the states under policy $f$ be represented by $\eta_f$ .Equivalently, the average cost can be written in terms of the $\eta_f$ as:
\alns{
J_f = \E_{s\sim\eta_f}\sbr{c_f\pbr{s}}.
}
However, this formulation is predicated on the assumption that the nominal transition kernel representing the underlying environment is error-free. An approach to robust MDPs to account for model uncertainties is to minimize the worst-case average cost over a $\mathsf{KL}$-ball around the nominal model:
\alns{
\sup_{Q:\E_{s\sim\eta_Q}\pbr{D_{\mathsf{KL}}\pbr{Q(s,\cdot)\|\P_f(s,\cdot)}}\leq \beta} \E_{s\sim\eta_Q}\sbr{c_f\pbr{s}},
}
where $D_{\mathsf{KL}}$ denotes the Kullback-Leibler divergence, and $\beta > 0$ is the radius of the $\mathsf{KL}$-ball. The above objective considers the largest average cost across a ball of distributions centered around the nominal transition kernel corresponding to each policy. This is the robust MDP objective considered in this paper. The dual formulation of the robust MDP objective is:
\alns{
\sup_{Q\ll\P_f} \E_{s\sim\eta_Q}\sbr{c_f\pbr{s}} - \frac{1}{\alpha}\E_{s\sim\eta_Q}\sbr{D_{KL}\pbr{Q(s,\cdot)\|\P_f(s,\cdot)}},
}
where the constant $\alpha = \alpha(\beta) > 0$ depends on $\beta$ and $\ll$ denotes absolute continuity. Using the Donsker-Varadhan variational formula and Collatz–Wielandt formula, it can be shown that optimizing the robust MDP objective is equivalent to minimizing 
\aln{ 
\Lambda_f(\alpha) = \lim_{t\to\infty} \frac{1}{t} \ln\left({\E\sbr{\exp\pbr{\sum_{k=0}^{t-1} \alpha c_f(s_k)} } }
\right),
\label{lambda}
}
where the expectation is taken with respect $\P_f$. The existence of the above limit is a consequence of the Perron-Frobenius Theorem, whose details can be found in \cite{mohar}, \cite{basu}. $\Lambda_f(\alpha)$ is known as the risk sensitive average cost. Similar to $J_f$, the value of $\Lambda_f(\alpha)$ does not depend on the initial state $s_0$. The level of risk aversion is quantified by $\alpha$, often referred to as the risk factor. Higher values of $\alpha$ correspond to increased risk aversion. Note that in the limit as $\alpha\to 0$, the risk-sensitive average cost converges to the risk neutral average cost, i.e., $$\lim_{\alpha \to 0} \Lambda_f(\alpha) = J_f.$$ For simplicity, from now on, we fix $\alpha > 0$ and write $\Lambda_f$ instead of $\Lambda_f(\alpha)$.

Analogous to the risk-neutral average cost formulation, the objective \eqref{lambda} can be expressed as the solution to a Bellman Equation \eqref{mbe}, but one that is in multiplicative form: 
\aln{ 
e^{\Lf}e^{V_f(i)} &= e^{\alpha c_f(i)} \sum_{j\in\S} \P_f(j|i) e^{V_f(j)}, \quad \forall \ i \in \S.\label{mbe}}
The relative value function $e^{V_f}$ is the eigenvector corresponding to the Perron-Frobenius eigenvalue $e^{\Lf}$ associated with the non negative matrix $[e^{\alpha c_f(i)}\P\pbr{j|i,f(i)}]_{i,j}$. From the above form of the multiplicative Bellman Equation, we can see that the value function is unique up to a multiplicative constant, thus being referred to as the relative value function.

The multiplicative Bellman operator corresponding to a policy $f$, is an operator $\T_f:\mathbb{R}_+^{|\S|}\to \mathbb{R}_+^{|\S|}$ defined as: 
\begin{equation*}
    \T_f e^{V}(i) = e^{\alpha c_f(i)}\sum_{j\in\S}\P_f(j|i)e^{V(j)}.
\end{equation*}
The optimal multiplicative Bellman operator $\T:\mathbb{R}_+^{|\S|}\to \mathbb{R}_+^{|\S|}$ is defined as: 
\begin{equation}
    \T e^{V}(i) = \min_{f\in\Pi}\T_f e^{V}(i),\qquad \forall i\in\S.
    \label{optimalmbo}
\end{equation}
The optimal risk sensitive average cost is defined as the minimum risk sensitive average cost across all policies, i.e.,
\aln{ 
\Lambda^* = \min_{f\in\Pi} \Lf = \min_{f\in\Pi} \lim_{t\to \infty} \frac{1}{t} \ln\pbr{ \E\sbr{ \exp\pbr{\sum_{k=0}^{t-1}\alpha c_f(s_k)} \bigg|\,  s_0 = i }}.
}
\noindent Let $f\in\Pi$ denote the deterministic policy for which $\Lf = \Lambda^*$, and let $e^{V^*} = e^{V_f}$ denote its relative value function. It can be shown that the pair $(\Lambda^*,e^{V^*})$ is the unique solution (up to multiplicative constant of $e^{V^*}$) to the following equation:
\begin{equation}\label{acoe}
    e^{\Lambda^*} e^{V^*(i)} =   \min_{f\in\Pi} e^{\alpha c_f(i)}\sum_{j\in\S}\P_f(j|i)e^{V^*(j)}, \quad 
    \forall \ i \in \S.
\end{equation}

\subsection{Assumptions}
Our assumptions on the model and policy class are summarized below. 
\begin{assumption}
    We assume that:
    \begin{itemize}
        \item The state space and the action space are finite.
        \item The one-step cost associated with each state-action pair $(s, a)\in\S\times\A$ is deterministic and bounded.
        \item The Markov process associated with each deterministic policy $f\in\Pi$ is irreducible.
    \end{itemize}
\end{assumption}
The irreducibility condition is somewhat restrictive. However, in MDPs where this condition is not satisfied, one can replace $\P$ with $\Tilde{\P} = (1-\epsilon)\P + \epsilon \boldsymbol{1}\boldsymbol{1}^\top$ where $\boldsymbol{1}$ is the all-one column vector and $\epsilon > 0$ is a fixed constant. This leads to an $O(\epsilon)$ suboptimality. 
While aperiodicity of the Markov chain associated with each policy is not an assumption, it is a relatively straightforward to transform the original MDP into a new MDP where every policy induces an aperiodic Markov chain. In the literature, such a transformation is referred to as the aperiodicity or the lazy-chain transformation. This transformation is essential to guarantee the convergence of the algorithm under consideration in this paper and is therefore described below.

Fix a constant $\kappa\in(0,1)$ and transform the dynamics of the MDP as follows:
\begin{itemize}
    \item The transformed cost is given by:
    \begin{equation*}
        d_f(i) = \frac{1}{\alpha}\log((1-\kappa)e^{\alpha c_f(i)}+\kappa),\qquad \forall i \in\S.
    \end{equation*}

    \item The transformed transition probabilities are given by:
    \begin{equation*}
        \Q(j|i,a) =\frac{(1-\kappa)e^{\alpha c(i,a)}\P(j|i,a)+\kappa\boldsymbol{1}(i=j)}{(1-\kappa)e^{\alpha c(i,a)}+\kappa},\qquad\forall(i,a)\in\S\times\A,
    \end{equation*}
    where $\boldsymbol{1}(i=j)$ is the indicator function. For any policy $f\in\Pi$, $\Q_f(j|i)$ denotes the probability of moving to state $j\in\S$ from state $i\in\S$ upon taking action $f(i)$.
\end{itemize}
Notice that for all $(i,a)\in\S\times\A$, we have $\Q(i|i,a)\geq\frac{\kappa}{(1-\kappa)e^{\alpha \overline{c}}+\kappa}>0$. In particular, the probability of staying in the same state under all policies is non-zero.  This aids in diminishing the mixing time of the MDP. 
The following theorem from \cite{cadena} establishes a one-to-one correspondence between the optimal risk sensitive average cost and the associated relative value function in the original MDP and the transformed MDP. Hence, finding an optimal policy for the transformed dynamics is equivalent to finding an optimal policy for the original MDP. 
\begin{theorem}\label{theorem2}
Given $\kappa\in(0,1)$, we have the followings:
\begin{enumerate}
    \item Given $(\Lambda^*,e^{V^*})$ satisfies \eqref{acoe}, define
    \begin{equation*}
        \tilde{\Lambda}^*=\log((1-\kappa)e^{\Lambda^*}+\kappa).
    \end{equation*}
    
    Then $(\tilde{\Lambda}^*,e^{V^*})$ solves the following multiplicative Bellman equation:
    \begin{equation}\label{acoe1}
    e^{\tilde{\Lambda}^*} e^{V^*(i)} = \min_{f\in\Pi} e^{\alpha d_f(i)}\sum_{j\in\S}\Q_f(j|i)e^{V^*(j)}, \quad 
    \forall \ i \in \S.    
\end{equation}
    \item Conversely, given $(\tilde{\Lambda}^*,e^{V^*})$ satisfies \eqref{acoe1}, 
    then
    \begin{equation*}
        e^{\tilde{\Lambda}^*}\geq\kappa.
    \end{equation*}
    
    Define 
    \begin{equation}\label{ogcost}
        {\Lambda}^*=\log\left(\frac{e^{\tilde{\Lambda}^*}-\kappa}{1-\kappa}\right).
    \end{equation}
    
    Then the pair $(\Lambda^*,e^{V^*})$ satisfies \eqref{acoe}. 
    \item Further, the transformed and original problems possess the same optimal policies.
\end{enumerate}
\end{theorem}

\subsection{Algorithm}
The modified policy iteration algorithm for risk sensitive exponential cost MDPs which have been transformed using the aperiodicity transformation is stated in Algorithm~\ref{Alg: MPI}. The algorithm takes as input a sequence of natural numbers $(m_i: i \in \mathbb{N})$ such that $m_i \geq 1$ and a vector $V_0' \in \R^n$ such that $\sum_{i\in\S}e^{V_0'(i)} = 1$.  The $m_i$ that are input in the algorithm can vary for every iteration $i$ but are predetermined and not random. The policy evaluation step becomes more accurate as the value of $m_i$ increases.

\begin{algorithm}
\caption{Risk Sensitive Modified Policy Iteration}\label{Alg: MPI}
\begin{algorithmic}[1]
\Require $(m_i: i \in \mathbb{N})$, $V'_0$. 
\For{$k = 0,1,\cdots$}
\State Set $f_{k+1}(i) = \argmin_{f\in\Pi} e^{\alpha d_f(i)} \sum\limits_{j\in\S} \Q({j\mid i, f(i)}) e^{V'_k(j)} \quad \forall i \in \S$ \Comment{Policy Improvement} \newline
\indent Define $\pbr{\T_{f_{k+1}} e^{V'_k}}(i) = e^{\alpha d_{f_{k+1}}(i)} \sum\limits_{j\in\S} \Q({j \mid i, f_{k+1}(i)}) e^{V'_k(j)} $ 
\State $e^{V_{k+1}(i)} \gets \pbr{\T_{f_{k+1}}^{m_k} e^{V'_k}}(i)$ for all $i \in \S$. \Comment{Partial Policy Evaluation} 
\State $e^{V'_{k+1}(i)} \gets \frac{e^{V_{k+1}(i)}}{\sum_i e^{V_{k+1}(i)}}$ for all $i \in \S$ \Comment{Normalization}
\EndFor
\State{\bfseries return} $(e^{V_k'}:k\geq 0)$ and $(f_k:k\geq 1)$
\end{algorithmic}
\end{algorithm}

\section{Results}\label{results}

The main result of the paper shows that Algorithm \ref{Alg: MPI} converges and also provides finite-time performance guarantees.
\begin{theorem}\label{thm_convg}
Recall the definition of the optimal multiplicative Bellman operator in \eqref{optimalmbo}. Let $e^{V'_n}$ and $f_n$, for all $n\in\N$ be the outputs of Algorithm \ref{Alg: MPI}. Then there exist $\gamma$ and $k\in\N$ such that $0<\gamma<1$ and for each $n$, the iterates of risk sensitive modified policy iteration satisfy:
\begin{equation}
    \left(\max_{i\in\S}\frac{\T e^{V'_n(i)}}{e^{V'_n (i)}} - e^{\Tilde{\Lambda}^*} \right) \leq \pbr{1-\gamma}\pbr{\max_{i\in\S}\frac{\T e^{V'_{n-k} (i)}}{e^{V'_{n-k} (i)}} - e^{\Tilde{\Lambda}^*}},
\end{equation}
where $e^{\Tilde{\Lambda}^*}$ is defined in \eqref{acoe1}.
Consequently, the risk sensitive average cost iterates converge to $e^{\Tilde{\Lambda}^*}$, i.e.
\aln{ 
\lim_{n\to\infty} e^{\Tilde{\Lambda}_{f_{n}}} = e^{\Tilde{\Lambda}^*}.
}
\end{theorem}
Proof can be found in Subsection \ref{subsection:proof}.

We will now outline the motivation behind our proof, highlight its distinctions from the proofs related to the convergence of other conventional dynamic programming techniques, and provide an overview of its key components. 

\subsection{Key Ideas and Proof Outline}
Compared to dynamic programming techniques such as value iteration or policy iteration, the convergence analysis of modified policy iteration in risk sensitive MDPs faces additional challenges due to partial policy evaluation in step 3 of algorithm \ref{Alg: MPI}. In this section we present the key ideas in our technique that circumvent these challenges. 

In the context of value iteration, it is well known that the consecutive value function iterates possess a span-seminorm contraction property (\cite{biel}, \cite{borkar1}). More precisely, let $g_0,h_0\in\R^{n}$. Then there exist constants $\tau,k,r$ such that  $0<\tau<1$, and $ \mathbb{N} \ni k,r <\infty $ such that
\aln{
     \text{sp}\pbr{g_k-h_k} \leq \tau^r\text{sp}\pbr{g_0-h_0},
\label{eqspan}}
where the span of a vector $v$ is defined as $\text{sp}(v) = \max_i v(i) - \min_i v(i)$ and
\aln{
g_k(i) = \min_{f\in\Pi}\cbr{\alpha d_f(i)+\ln\pbr{\sum_{j\in\S} \Q(j|i,f(i))e^{g_{k-1}(y)}}},
\label{eqg}}
\aln{
h_k(i) = \min_{f\in\Pi}\cbr{\alpha d_f(i)+\ln\pbr{\sum_{j\in\S} \Q(j|i,f(i))e^{h_{k-1}(y)}}}.
\label{eqh}}
Note that \eqref{eqg} and \eqref{eqh} are value iteration updates corresponding to value function vectors $g_{k-1}$ and $h_{k-1}$ respectively. A similar contraction in the sup norm is satisfied in the discounted-cost setting, where the discount factor serves as the source of contraction. The span contraction factor $\tau$ in \eqref{eqspan} is the Birkhoff coefficient of a matrix, which is dependent on the risk sensitivity factor $\alpha,$ the single step cost vector and the underlying probability transition kernel. It is worth noting that irreducible and aperiodic Markov chain under any policy are needed to ensure $\tau<1.$ To the best of our knowledge, all prior works which obtain finite time bounds for dynamic programming algorithms for risk sensitive MDPs seem to use these assumptions. When $h_0$ is the optimal risk sensitive value function, \eqref{eqspan} ensures convergence of the iterates $g_k$, since $h_k$ are all identical upto a multiplicative constant (\cite{biel}). The inequality \eqref{eqspan} is a consequence of consecutive iterates of $g_k,h_k$ being related through the optimal multiplicative Bellman operator, which by virtue of its definition possesses a monotonicity property. This property is also central to the convergence analysis of risk-sensitive value iteration discussed in \cite{cadena}, even though they provide asymptotic convergence bounds by proving that the difference between two consecutive iterates tends to zero. 

However, such monotonicity fails to hold when partial policy evaluation is involved. In the context of risk-sensitive modified policy iteration (MPI), at each iteration $k$, the iterate $e^{V'_{k+1}}$ is derived by applying the multiplicative Bellman operator associated with the policy $f_{k+1}$ to the preceding iterate $e^{V'_k}$, a total of $m_k$ times. In this computation, only the first application of the Bellman operator corresponds to the optimal Bellman operator acting on $e^{V'_k}$, since $\T_{f_{k+1}}e^{V'_k} = \T e^{V'_k}$ (from step 2 of algorithm \ref{Alg: MPI}). This computation is followed by applying $\T_{f_{k+1}}$ repeatedly for a total of $m_k-1$ times - which is the step that distinguishes the algorithm from value iteration. Hence the monotonicity arguments leveraged for the finite time convergence analysis of value iteration break down in the context of modified policy iteration. As a corollary to this observation, one can even view the modified policy iteration algorithm as a broader form of value iteration, since setting $m_k=1$ for all $k\in\N$ yields the value iteration algorithm. Consequently, the convergence analysis presented in this paper also implies finite-time convergence bounds for value iteration

The risk sensitive policy iteration algorithm can be obtained by replacing step 3 in Algorithm~\ref{Alg: MPI} with complete policy evaluation (this step involves determining the Perron Frobenius eigenvector of a function of the transition kernel). In this setting, there exists a state $i$ such that,
\alns{
\pbr{\T_{f_{k+1}}e^{V_{f_k}}}(i) &:= e^{\alpha d_{f_{k+1}}(i)} \sum\limits_{j\in\S} \Q({j \mid i, f_{k+1}(i)}) e^{V_k(j)}\\
&\stackrel{\text{(a)}}{<} e^{\alpha d_{f_{k}}(i)} \sum\limits_{j\in\S} \Q({j \mid i, f_{k}(i)}) e^{V_k(j)}\\ & = e^{\Tilde{\Lambda}_{f_{k}}}e^{V_{f_k}(i)},
}
where (a) follows from policy improvement. The above strict inequality (a) is true for some state $i\in\S$ if the algorithm is yet to converge. Upon repeatedly applying $\T_{f_{k+1}}$ operator in the above inequality, in the infinite limit of its application, we obtain that $e^{\Tilde{\Lambda}_{f_{k+1}}}<e^{\Tilde{\Lambda}_{f_{k}}}$. In the setting of finite state-action spaces with deterministic policies, this monotonicity of risk sensitive average costs is sufficient to ensure convergence. However, this proof technique relies on complete policy evaluation and once again is ineffective in the context of risk sensitive modified policy iteration due to step 3 in Algorithm \ref{Alg: MPI}. Despite modified policy iteration lacking value function iterates that are related through contracting span seminorms, or risk sensitive average cost iterates that are monotonically decreasing, the ratio $\frac{\T e^{V'_k(i)}}{e^{V'_k(i)}}$ exhibits interesting properties that can be used to circumvent the above mentioned issues. The inspiration for our proof comes from a related proof of convergence of modified policy iteration for average cost MDPs in \cite{vanderwal}. However, unlike \cite{vanderwal} the fact that we have to deal with ratios instead of differences presents new challenges. We will soon show that a function of the ratio $\frac{\T e^{V'_k(i)}}{e^{V'_k(i)}}$ possesses monotonicity properties and contraction properties which aid in the convergence analysis. We now present an outline of the proof below.

\textbf{Step 1: Characterizing the impact of applying the optimal multiplicative Bellman operator once}  
In order to take advantage of the contraction and monotonicity properties pertaining to the ratio $\frac{\T e^{V'_k(i)}}{e^{V'_k(i)}}$, it is necessary to be able to characterize its relationship to the risk sensitive average cost associated with the optimizing policy associated with $e^{V'_k}$, as well as its relationship to the risk sensitive optimal cost obtained as the solution to the optimal multiplicative Bellman optimality equation \eqref{acoe1}. This result is captured in Lemma~\ref{lem_VI}, which can be utilized to extend convergence properties of  $\frac{\T e^{V'_k(i)}}{e^{V'_k(i)}}$ to the associated risk sensitive average costs. 

\textbf{Step 2: Establishing an appropriate monotonicity}
We show that $\max_{i\in\S}\frac{\T e^{V'_k(i)}}{e^{V'_k(i)}}$ is a monotonically decreasing sequence. Note that from \eqref{acoe1} we know that $\frac{\T e^{V^*(i)}}{e^{V^*(i)}} = e^{\Tilde{\Lambda^*}}$ for all $i\in\S$. Even though $e^{V'_k}$ are not value functions in the sense of \eqref{mbe}, the monotonically decreasing $\max_{i\in\S}\frac{\T e^{V'_k(i)}}{e^{V'_k(i)}}$ alludes to convergence of the algorithm. Lemma \ref{lemm_mono} contains the proof of this observation. It is important to note that this monotonicity is observed only for $\max_{i\in\S}\frac{\T e^{V'_k(i)}}{e^{V'_k(i)}}$ and not for $\min_{i\in\S}\frac{\T e^{V'_k(i)}}{e^{V'_k(i)}}$. If $\min_{i\in\S}\frac{\T e^{V'_k(i)}}{e^{V'_k(i)}}$ were to be monotonically increasing as well, we would have a span contraction for these iterates which would have simplified the proof. 

\textbf{Step 3: Identifying a source of contraction} In order to establish finite time convergence bounds it is necessary to identify a source of contraction between subsequent iterates. In discounted cost MDPs, it is achieved by the discount factor. In average cost MDPs it is achieved by the least value of the stationary measure across all states and policies. In the risk sensitive value iteration setting, it is the Birkhoff's coefficient of a model-dependent matrix. 

Here, two conditions need to be satisfied in order to obtain a valid source of contraction: (i) aperiodicity and irreducibility of the transition kernel under all policies and (ii) uniform boundedness of all iterates $e^{V'_k}$ away from zero. Condition (i) is satisfied by virtue of our modeling assumptions. To ensure condition (ii), we use a specific form of normalization (step 4 of Algorithm \ref{Alg: MPI}.) In the case of risk-neutral average-cost modified policy iteration, the normalization step generally involves subtracting the value function at some fixed state from the rest of the states. This ensures that the value function iterates do not diverge with repeated execution of the algorithm. However, a similar normalization trick where one divides the value function for all states by the value function of one state does not seem to work for risk sensitive modified policy iteration; therefore we divide by the sum of the value function of all states. This allows us to ensure the value functions do not diverge and that they are uniformly bounded away from zero.

Lemma \ref{lem2} utilizes condition (i) and Lemma \ref{lemm_bound} proves the boundedness of $e^{V'_k(i)}$ to establish a source of contraction necessary for convergence analysis in the main proof. 

\textbf{Step 4: Convergence of risk sensitive modified policy iteration} Lemma \ref{lem_VI} and the monotonicity of $\max_{i\in\S}\frac{\T e^{V'_k(i)}}{e^{V'_k(i)}}$ along with the source of contraction are utilized to prove the exponential convergence of $\max_{i\in\S}\frac{\T e^{V'_k(i)}}{e^{V'_k(i)}}$ to $e^{\Tilde{\Lambda^*}}.$ Lemma \ref{lem_VI} is further leveraged to prove the convergence of $e^{\Tilde{\Lambda_{f_k}}}$ to $e^{\Tilde{\Lambda^*}}$ and $\min_{i\in\S}\frac{\T e^{V'_k(i)}}{e^{V'_k(i)}}$ to $e^{\Tilde{\Lambda^*}}$ thereby proving the convergence of modified policy iteration.

\section{Proofs}\label{proofs}
We present the lemmas and their proofs and also provide the main proof presented in the proof outline in the previous section.  
\subsection{Step 1: Characterizing the impact of applying the optimal multiplicative Bellman operator once}
Let $g_n$ be defined as 
\aln{ 
g_n(i) &:= \frac{\T e^{V_n'}(i) }{e^{V_n'}(i)},    \label{def1} 
}
and set $u_n$ and $\ell_n$ as
\aln{
u_n &:= \max_{i\in\S} \pbr{ g_n(i) }  \label{def2}\\
\ell_n &:= \min_{i\in\S} \pbr{g_n(i)}.
\label{def3}}

\begin{lemma}\label{lem_VI}
Let $\Tilde{\Lambda}^*$ be the optimal risk sensitive average cost associated with the MDP considered in Algorithm \ref{Alg: MPI}. Then $\forall n>0$:
\aln{ 
\ell_n \leq e^{\Tilde{\Lambda}^*} \leq e^{\Tilde{\Lambda}_{f_{n+1}}} \leq u_n.
}
\end{lemma}
\proof{Proof.}
    \begin{itemize}
        \item $u_n \geq e^{\Tilde{\Lambda}_{f_{n+1}}}$. Recall by definition that
        \alns{ 
        u_n &= \max_i \pbr{\frac{\T e^{V_n'(i)}}{e^{V_n'(i)}}} \\ &\geq
        \frac{\T e^{V_n'(i)}}{e^{V_n'(i)}} \\ &= \frac{\T_{f_{n+1}} e^{V_n'}(i)}{e^{V_n'}(i)} \quad \forall i\in\S.
        }

        Therefore, we have
        \alns{ 
        u_n e^{V_n'(i)} \geq e^{d_{f_{n+1}}(i)}\sum_{j\in\S} \Q_{f_{n+1}} \pbr{ j\mid i} e^{V_n'(j)},
        }
        
        and thus
        \alns{ 
        u_n e^{V_n'} \geq \Qt_{f_{n+1}} \pbr{e^{V_n'}}.
        }
        
        where ${\Qt_{f_{n+1}}}{(i,j)} = e^{d_{f_{n+1}}}(i)\Q_{f_{n+1}} \pbr{ j\mid i}.$ 
    
        Consequently, it follows that
        \alns{ 
        u_n \Qt_{f_{n+1}}^k \pbr{e^{V_n'}} \geq \pbr{\Qt_{f_{n+1}}}^{k+1} e^{V_n'}
        }
        
        and
        \alns{ 
        \frac{u_n}{e^{\Tilde{\Lambda}_{f_{n+1}}}} \pbr{\frac{\Qt_{f_{n+1}}}{e^{\Tilde{\Lambda}_{f_{n+1}}}}}^k e^{V_n'} \geq \pbr{\frac{\Qt_{f_{n+1}}}{e^{\Tilde{\Lambda}_ {f_{n+1}}}}}^{k+1} e^{V_n'}.
        }
      
        Note that
    
        \alns{ 
        \lim_{k\to\infty} \frac{u_n}{e^{\Tilde{\Lambda}_{f_{n+1}}}} \pbr{\frac{\Qt_{f_{n+1}}}{e^{\Tilde{\Lambda} f_{n+1}}}}^k e^{V_n'} \geq \lim_{k\to\infty} \pbr{\frac{\Qt_{f_{n+1}}}{e^{\Tilde{\Lambda}_{f_{n+1}}}}}^{k+1}e^{V_n'}
        }.
    
         Since $\Q$ satisfies the conditions of the Perron-Frobenius Theorem, by definition, so does $\widetilde{Q}$. The vector $e^{V'_n}$ consists of all positive elements and hence, in the limit of $k\longrightarrow{\infty}$, due to power iteration, the following holds true:
      
        \alns{ 
        \frac{u_{n}}{e^{\Tilde{\Lambda}_{f_{n+1}}}} z \geq z.
        }
        
        As $\Qt_{f_{n+1}}$ and $e^{V_n'}$ are both non negative, the resulting $z$ is a non-zero vector containing non negative elements. Hence we obtain,
        \alns{ 
        u_n \geq e^{\Tilde{\Lambda}_{f_{n+1}}},
        }
        
        as desired.
    
        \item $e^{\Tilde{\Lambda}^*} \leq e^{\Tilde{\Lambda}_{f_{n+1}}}$. 
        Recall that
        \alns{ e^{\Tilde{\Lambda}^*} e^{V^*(i)} &= \min_{f\in\Pi} e^{\alpha d_f(i)} \sum_{j\in\S} \Qr{j\mid i, f(i)} e^{V^*(j)}
        \\
        &\leq e^{\alpha d_{f_{n+1}} (i)} \sum_{j\in\S} \Qr{j\mid i,f_{n+1}(i) } e^{V^*(j)}.
        }
        
        Therefore, we have
        \alns{ 
        e^{V^*(i)} &\leq \frac{e^{ \alpha d_{f_{n+1}}(i)} \sum_{j\in\S} \Qr{j \mid i,f_{n+1}(i)} e^{V^*(j)} }{e^{\Tilde{\Lambda}^*}}\\
        &= \frac{e^{\alpha d_{f_{n+1}}(i) } \E\sbr{ e^{V^*(x_1)} \mid x_0 = i, f_{n+1}} }{e^{\Tilde{\Lambda}^*}}\\
        & \leq \frac{e^{\alpha d_{f_{n+1}}(i)}}{e^{\Tilde{\Lambda}^*}} \cdot \E\sbr{ \frac{e^{\alpha d_{f_{n+1}}(x_1)}}{e^{\Tilde{\Lambda}^*}} \cdot \E\sbr{e^{V^*(x_2)}\mid x_1,f_{n+1}} \Big| \ x_0=i,f_{n+1}}.
        }
    
        Iterating, we get
        \alns{ 
        e^{V^*(x_0)} \leq \E\sbr{ \frac{ e^{\sum_{i=0}^{k-1} \alpha d_{f_{n+1}}(x_i) }V^*(x_k)}{ \pbr{e^{\Tilde{\Lambda}^*}}^k} \Bigg| \ x_0}.
        }
        
        Since $V^*(x_k) \leq M < \infty$ for all $x_k \in \S$, we have
        \alns{ 
        k \cdot \Tilde{\Lambda}^* + V^*(x_0) & \leq \ln\pbr{ \E_{x_0} \sbr{e^{\alpha\sum_{i=0}^{k-1} d(x_i,f_{n+1}(x_i))} }} + \ln(M),
        }
        
        or equivalently
        \alns{ 
        \Tilde{\Lambda}^* + \frac{V^*(x_0)}{k} \leq \frac{1}{k} \ln\pbr{ \E_{x_0} \sbr{e^{\alpha\sum_{i=0}^{k-1} d(x_i,f_{n+1}(x_i))}}} + \frac{M}{k}.
        }
        
        As $k \to \infty$, we obtain
        \alns{ 
        \Tilde{\Lambda}^* \leq \Tilde{\Lambda}_{f_{n+1}}.
        }
        
        By monotonicity of the exponential function, we conclude as desired that
        \alns{ 
        e^{\Tilde{\Lambda}^*} \leq e^{\Tilde{\Lambda}_{f_{n+1}}}.
        }
        \item $\ell_n \leq e^{\Tilde{\Lambda}^*}$. Recall that $e^{\Tilde{\Lambda}^*}$ satisfies the following equation:
        \alns{ 
        e^{\Tilde{\Lambda}^*} e^{V^*} = \min_{f\in\Pi} e^{\alpha d_{f}} \Q_{f} \pbr{e^{V^*}}.
        }
        
        Let the minimising policy be $f^*$. We have
        \alns{ 
        \ell_n &= \min_i \frac{\T e^{V_n'}(i)}{e^{V_n'(i)}}\\ &\leq \frac{\T e^{V_n'}(i)}{e^{V_n'(i)}} \\ &= \frac{\T_{f_{n+1}} e^{V_n'}(i)}{e^{V_n'(i)}} \\&\leq \frac{\T_{f^*} e^{V_n'(i)}}{e^{V_n'}(i)}.
        }
        
        Therefore,
        \alns{ 
        \ell_n e^{V_n'(i)} \leq e^{\alpha d_{f^*}(i)} \sum_{j\in\S} \Qr{j \mid i,f^*(i)}e^{V_n'(j)} = \pbr{\Qt_{f^*} \pbr{e^{V_n'}}}(i).
        }
        
        It follows that
        \alns{ 
        \ell_n \Qt_{f^*}^k \pbr{e^{V_n'}} \leq \Qt_{f^*}^{k+1} e^{V_n'}.
        }
        
        Similarly, we have
        \alns{
        \frac{\ell_n}{e^{\Tilde{\Lambda}^*}} \pbr{\frac{\Qt_{f^*}}{e^{\Tilde{\Lambda}^*}}}^k e^{V_n'} \leq \pbr{\frac{\Qt_{f^*}}{e^{\Tilde{\Lambda}^*}}}^{k+1} e^{V_n'}.
        }
        
        As a consequence of Perron-Frobenius Theorem, it follows that
        \alns{ 
        \ell_n \leq e^{\Tilde{\Lambda}^*},
        }
        
        as desired. This concludes the proof.    \hfill\Halmos
    \end{itemize}
\endproof

The above Lemma is crucial to the proof of convergence of modified policy iteration. A similar relation would hold for the reward maximization problem: $\ell_n \leq e^{\Tilde{\Lambda}_{f_{n+1}}} \leq e^{\Tilde{\Lambda}^*} \leq u_n$. Such a relation can be obtained in the context of risk-neutral average cost (\cite{vanderwal}) as well. The multiplicative nature of Bellman operator combined with the exponential cost formulation, necessitates a different proof idea which hinges on the careful utilization of the Perron-Frobenius Theorem.


\subsection{Step 2: Establishing an appropriate monotonicity property}

\begin{lemma}\label{lemm_mono}
The sequence $u_n$ is non-increasing, i.e. $u_n \leq u_{n-1}$ for all $n$.
\end{lemma}
\proof{Proof.}
{
Recall that 
\alns{ 
u_n = \max_i \frac{\pbr{\T e^{V_n'}}(i)}{e^{V_n'}(i)}.
}

Let $x^* = \argmax_i \frac{\pbr{\T e^{(V_n')}}(i)}{e^{V_n'}(i)}$, so that
\alns{ 
u_n & = \frac{\pbr{\T e^{(V_n')}}(x^*)}{e^{V_n'}(x^*)}  = \frac{\pbr{ \Qt_{f_{n+1}}\pbr{e^{(V_n')}}(x^*)}}{e^{V_n'}(x^*)}  \leq \frac{\pbr{ \Qt_{f_{n}}\pbr{e^{(V_n')}}(x^*)}}{e^{V_n'}(x^*)}.
}

Since $e^{V_n'(j)} = \frac{e^{V_n(j)}}{\sum_{j\in\S} e^{V_n(j)}}$,
it follows that
\alns{ 
u_n &\leq \frac{\Qt_{f_n}\pbr{ e^{V_n}(x^*)}}{e^{V_n(x^*)}} \\
&= \frac{\Qt_{f_n} \Qt_{f_n}^{m_{n-1}} e^{V'_{n-1}} (x^*)}{ \Qt_{f_n}^{m_{n-1}} e^{V'_{n-1}}(x^*)}  \\
&= \frac{ \pbr{ \Qt_{f_n}^{m_{n-1}} \pbr{\Qt_{f_n} e^{V'_{n-1}}}}(x^*)}{\pbr{\Qt_{f_n}^{m_{n-1}}e^{V'_{n-1}}}(x^*)}.
}

Let $e^{W'_{n-1}} = \Qt_{f_n} e^{V'_{n-1}}$ and $H_n = \Qt_{f_n}^{m_{n-1}}$. It then follows that
\alns{ 
u_n &\leq \frac{ \pbr{H_n e^{W'_{n-1}}}(x^*)}{\pbr{H_n e^{V'_{n-1}}}(x^*)}\\
&= \frac{\sum_{j\in\S} H(j \mid x^*) e^{W'_{n-1}(j)}}{\sum_{j\in\S} H(j \mid x^*) e^{V'_{n-1}(j)}}.
}

Let $p = \argmax_i \frac{e^{W'_{n-1}(i)}}{e^{V'_{n-1}(i)}}$. Then, we have
\alns{ 
\frac{e^{W'_{n-1}(p)}}{e^{V'_{n-1}(p)}} \geq \frac{e^{W'_{n-1}(i)}}{e^{V'_{n-1}(i)}} \quad \forall \ i \in \S.
}

This yields
\alns{ 
u_n \leq \frac{ \sum_{j\in\S} \pbr{H(j\mid x^*) e^{V'_{n-1}(j)} \cdot \frac{e^{W'_{n-1}(p)}}{e^{V'_{n-1}(p)}}  } }{ \sum_{j\in\S} H(j\mid x^*) e^{V'_{n-1}(j)}}.
}

Therefore,
\alns{ 
u_n &\leq \frac{e^{W'_{n-1}(p)}}{e^{V'_{n-1}(p)}} \\&= \frac{\Qt_{f_n} \pbr{e^{V'_{n-1}}}(p)}{e^{V'_{n-1}(p)}} \\
 &= \max_i \frac{ \pbr{\T e^{V'_{n-1}}}(i)}{e^{V'_{n-1}(i)}} \\& = u_{n-1},
}

which establishes the desired monotonicity. \hfill\Halmos
}
\endproof

Value Iteration leads to monotonicity in $u_n$(non-increasing) and $\ell_n$(non-decreasing). This is a consequence of improving the policy at every iteration without any partial policy evaluation. This symmetric monotonicity  leads to an overall span contraction in the value function. However, due to partial policy evaluation in modified policy iteration, such a monotonicity is observed only for the maximum of the ratio of iterates, ie., $u_n$ (or $\ell_n$ in case risk sensitive reward maximization). Consequently, there need not be a span contraction for the value functions. Hence it is necessary to rely on arguments independent of span in order to prove algorithm convergence. This approach is delineated in the theorem below. 

\subsection{Step 3: Identifying a source of contraction}

A crucial component to the convergence of the algorithm is a source of contraction, which is obtained from any finite product of ergodic matrices.  
\begin{lemma}\label{lem2}
There exists a finite natural number $R$ such that for any sequence of policies $f_1$, $f_2$, $\cdots$, $f_R \in \Pi$,
\begin{equation}
    \min_{i,j\in \S} \Q_{f_1}\Q_{f_2}\cdots\Q_{f_R}(j|i) > 0.
\end{equation}
\end{lemma}
\proof{Proof.}
{
Given finite state and action spaces, the total number of deterministic policies are given by $|\Pi|=|\A|^{|\S|}$. Since every policy induces an irreducible Markov chain, $\P^{|\S|}_f(j|i)>0 \ \forall \ i,j\in\S, \forall \ f\in\Pi.$ When $R={|\A|}^{|\S|}+1$, there exists a policy that is repeated at least twice in the sequence. Hence, if $R=({|\S|}-1)\cdot {|\A|}^{|\S|} +1$, there exists a policy which is repeated at least ${|\S|}$ times. Since in the transformed model, under every policy the probability of staying in the same state is non-zero, there exists a non-zero probability of traversing from any state to any other state when $R\geq({|\S|}-1)\cdot {|\A|}^{|\S|} +1$ under any sequence of policies.
\begin{flushright}
 \vspace{-6mm}   \Halmos
\end{flushright}
}
\endproof


\begin{lemma}\label{lemm_bound}
Let $\max_k m_k < C$, where $m_k$ corresponds to the number of fixed point iterations performed during partial policy evaluation during the $k$th execution of the algorithm. Then, there exists $\beta$ such that $0< \beta <1$, 
\aln{ 
e^{V_m'(i)} > \beta > 0 \ \forall \ m \geq 0.
}
\end{lemma}
\proof{Proof.}
{
Since $e^{V_k'(i)} = \frac{e^{V_k(i)}}{\sum_{j\in\S} e^{V_k}(j)}$, it follows that $\sum_{j\in\S} e^{V_k'(j)} = 1$. We then have
\alns{
e^{V_k(i)} &= (T^{m_{k-1}}_{f_k} e^{V_{k-1}'})(i). 
}

Let $\underline{d}=\min_{f\in\Pi}d_f$ and $\overline{d}=\max_{f\in\Pi}d_f$. Then,
\alns{
T_{f_k}e^{V_{k-1}'}(i) &= e^{\alpha d_f(i)}\sum_{j\in\S}\Q(j|i,f_k(i))e^{V_{k-1}'(j)}\\
&\geq e^{\alpha\underline{d}}(\Q_{f_k} e^{V_{k-1}'})(i).
}

Iterating,
\alns{
(T^{m_{k-1}}_{f_k} e^{V_{k-1}'})(i) &\geq e^{m_{k-1}\alpha\underline{d}}(\Q^{m_{k-1}}_{f_k} e^{V_{k-1}'})(i) \\
&= \frac{e^{m_{k-1}\alpha\underline{d}}(\Q^{m_{k-1}}_{f_k} e^{V_{k-1}})(i)}{\sum_{j\in\S} e^{V_{k-1}(j)}},
}
\alns{
e^{V_k(i)} \geq \frac{e^{m_{k-1}\alpha\underline{d}}(\Q^{m_{k-1}}_{f_k} e^{V_{k-1}})(i)}{\sum_{j\in\S} e^{V_{k-1}(j)}}.
}

Further iterating, 
\alns{
e^{V_k(i)} \geq \frac{e^{\sum_{l=1}^k m_{k-l}\alpha\underline{d}}(\Q^{m_{k-1}}_{f_k}\Q^{m_{k-2}}_{f_{k-1}}\ldots\Q^{m_{0}}_{f_1} e^{V_{0}'})(i)}{\sum_{j\in\S} e^{V_{k-1}(j)}\sum_{j\in\S} e^{V_{k-2}(j)}\ldots\sum_{j\in\S} e^{V_{1}(j)}}.
}

From Algorithm \ref{Alg: MPI}, for a sufficiently large $k$, we have $\sum_{l=0}^{k-1} m_l>R$.
Defining $H_k = \Q^{m_{k-1}}_{f_k}\Q^{m_{k-2}}_{f_{k-1}}\ldots\Q^{m_{0}}_{f_1}$, Lemma \ref{lem2} yields $\eps = \min_{i,j} H_k(i\mid j) > 0$,
we continue the above sequence of inequalities:
\begin{align}\label{eq64}
   e^{V_k(i)} &\geq \frac{e^{\sum_{l=1}^k m_{k-l}\alpha\underline{d}} \sum_{j\in\S} \eps e^{V_0'(j)}}{\sum_{j\in\S} e^{V_{k-1}(j)}\sum_{j\in\S} e^{V_{k-2}(j)}\ldots\sum_{j\in\S} e^{V_{1}(j)}} \\
&= \frac{e^{\sum_{l=1}^k m_{k-l}\alpha\underline{d}}\eps }{\sum_{j\in\S} e^{V_{k-1}(j)}\sum_{j\in\S} e^{V_{k-2}(j)}\ldots\sum_{j\in\S} e^{V_{1}(j)}}. 
\end{align}

Similarly we obtain, 
\alns{
e^{V_k(i)} \leq \frac{e^{m_{k-1}\alpha\overline{d}}(\Q^{m_{k-1}}_{f_k} e^{V_{k-1}})(i)}{\sum_{j\in\S} e^{V_{k-1}(j)}}.
}

Further iterating, 
\alns{
e^{V_k(i)} \leq \frac{e^{\sum_{l=1}^k m_{k-l}\alpha\overline{d}}(\Q^{m_{k-1}}_{f_k}\Q^{m_{k-2}}_{f_{k-1}}\ldots\Q^{m_{0}}_{f_1} e^{V_{0}'})(i)}{\sum_{j\in\S} e^{V_{k-1}(j)}\sum_{j\in\S} e^{V_{k-2}(j)}\ldots\sum_{j\in\S} e^{V_{1}(j)}}.
}

This implies
\alns{
\sum_{i\in\S} e^{V_k(i)} \leq \frac{ne^{\sum_{l=1}^k m_{k-l}\alpha\overline{d}}}{\sum_{j\in\S} e^{V_{k-1}(j)}\sum_{j\in\S} e^{V_{k-2}(j)}\ldots\sum_{j\in\S} e^{V_{1}(j)}}.
}

Therefore,
\begin{equation}\label{68}
\frac{1}{\sum_{i\in\S} e^{V_k(i)}} \geq \frac{\sum_{j\in\S} e^{V_{k-1}(j)}\sum_{j\in\S} e^{V_{k-2}(j)}\ldots\sum_{j\in\S} e^{V_{1}(j)}}{ne^{\sum_{l=1}^k m_{k-l}\alpha\overline{d}}}.    
\end{equation}

Combining \eqref{eq64} and \eqref{68}, since $\forall l, \ m_l \geq 1 \ \text{and} \ m_l < C $
\begin{align*}
    \frac{e^{V_k(i)}}{\sum_{j\in\S} e^{V_k(j)}} &\geq \frac{e^{\sum_{l=1}^k m_{k-l}\alpha\underline{d}}\eps}{ne^{\sum_{l=1}^k m_{k-l}\alpha\overline{d}}} > \frac{e^{k\alpha\underline{d}}\eps}{ne^{kC\alpha\overline{d}}} > 0.
\end{align*}

We conclude that $e^{V_m'(i)} > \beta > 0$ for all $m$, where $\beta = \frac{e^{k\alpha\underline{d}}\eps}{ne^{kC\alpha\overline{d}}}$. \hfill \Halmos
}
\endproof

\subsection{Step 4: Convergence of risk sensitive modified policy iteration}
\label{subsection:proof}
We restate the theorem for convenience.
\begin{customthm}{2}
Let $g_n, u_n$ and $\ell_n$ be determined from Algorithm \ref{Alg: MPI} as per \eqref{def1}, \eqref{def2} and \eqref{def3} respectively. 

Then, $u_n$ converges exponentially fast, i.e. there exist $\gamma, k$ such that $0<\gamma<1$ and for each $n$:
$$\pbr{ u_n - e^{\Tilde{\Lambda}^*}} \leq \pbr{1-\gamma}\pbr{ u_{n-k} - e^{\Tilde{\Lambda}^*}}. 
$$ 

Consequently, the risk sensitive average cost iterates converge to $\Tilde{\Lambda}^*$, that is,
\aln{ 
\lim_{n\to\infty} u_n = \lim_{n\to\infty} \ell_n = e^{\Tilde{\Lambda}^*}.
}
\end{customthm}

\proof{Proof.}
{ 
By definition of $g_n$, we have
\alns{ 
g_n(i) &= \frac{ \T e^{V_n'}(i)}{ e^{V_n'(i)}} 
= \frac{ \pbr{\Qt_{f_{n+1}} e^{V_n'} }(i)}{ e^{V_n'}(i)} 
\stackrel{(a)}{\leq} \frac{ \pbr{ \Qt_{f_n} e^{V_n'}}(i)}{e^{V_n'}(i)}  \stackrel{(b)}{= } \frac{ \pbr{ \Qt_{f_n} e^{V_n}}(i) }{ e^{V_n(i)}},
}

where $\pbr{\Qt_{f_n} e^{V_n}}(i)= e^{\alpha d_{f_n}(i)}\sum_{j\in\S}\Q(j|i,f_n(i))e^{{V_n}(j)}$ (a) follows from the fact that $f_{n+1}$ is the minimizing policy, and (b) is due to $e^{V'_n(i)} = \frac{e^{V_n(i)}}{\sum_{j\in\S} e^{V_n(j)}}$. 

Using the definition of $e^{V_n}(i)$, we have
\alns{
g_n(i) & \leq \frac{ \pbr{ \pbr{\Qt_{f_n}} \pbr{ \Qt_{f_n}^{m_{n-1}}\cdot e^{V'_{n-1}}}}(i)}{ \pbr{\Qt_{f_n}^{m_{n-1}} e^{V'_{n-1}}}(i)} \\
& = \frac{\pbr{ \Qt_{f_n}^{m_{n-1}} \Qt_{f_n} e^{V'_{n-1}}}(i)}{\pbr{ \Qt_{f_n}^{m_{n-1}} e^{V'_{n-1}}}(i)} \\
& \leq \frac{\pbr{\Qt_{f_n}^{m_{n-1}} \cdot \Qt_{f_{n-1}} e^{V_{n-1}}}(i)}{\pbr{\Qt_{f_n}^{m_{n-1}} e^{V_{n-1}} }(i)} \\
& = \frac{\pbr{\Qt_{f_n}^{m_{n-1}} \Qt_{f_{n-1}}^{m_{n-2}} \Qt_{f_{n-1}} e^{V'_{n-2}}}(i)}{\pbr{ \Qt_{f_n}^{m_{n-1}} \Qt_{f_{n-1}}^{m_{n-2}} e^{V'_{n-2}}}(i)}.
}

Continuing the above for $k$ time steps, we get
\alns{ 
g_n \leq \frac{ \pbr{ \Qt_{f_{n}}^{m_{n-1}} \Qt_{f_{n-1}}^{m_{n-2}} \Qt_{f_{n-2}}^{m_{n-3}} \cdots \Qt_{f_{n-k+1}}^{m_{n-k}} \Qt_{f_{n-k+1}} e^{V'_{n-k}}  }(i)}{ \pbr{ \Qt_{f_{n}}^{m_{n-1}} \Qt_{f_{n-1}}^{m_{n-2}} \Qt_{f_{n-2}}^{m_{n-3}} \cdots \Qt_{f_{n-k+1}}^{m_{n-k}} e^{V'_{n-k}}}(i)}.
}

Let $H_{n,k} := \Qt_{f_{n}}^{m_{n-1}} \Qt_{f_{n-1}}^{m_{n-2}} \Qt_{f_{n-2}}^{m_{n-3}} \cdots \Qt_{f_{n-k+1}}^{m_{n-k}}$. 
From Lemma \ref{lem2}, we know that $\Q$ induces an irreducible Markov chain for any sequence of policies, i.e.,
\alns{ 
\exists \ R < \infty \text{ such that } \forall \ \pi_1,\cdots,\pi_R \in \Pi \colon \pbr{ \Q_{\pi_1} \Q_{\pi_2} \cdots \Q_{\pi_R}}(j|i) > 0 \quad \forall i,j.
}

The number of time steps $k$ is determined such that $m_{n-1} + m_{n-2} + \cdots + m_{n-k} \geq R$. This implies that $H_{n,k}(j\mid i) > 0$ for all $i,j$.

Let $e^{W'_{n-k}}:=\Qt_{f_{n-k+1}} e^{V'_{n-k}}$. We have
\alns{ 
g_n(i) &\leq \frac{ \pbr{H_{n,k} e^{W'_{n-k}}}(i)}{ H_{n,k} e^{V'_{n-k}}(i)} \\
& = \frac{ \sum_{j\in\S} H_{n,k} (j\mid i) e^{W'_{n-k}(j)}}{\sum_{\ell\in\S} H_{n,k} \pbr{\ell\mid i} e^{V'_{n-k}(\ell)}} \\
& = \frac{ \sum_{j\in\S} \pbr{ H_{n,k}(j\mid i) e^{W'_{n-k}(j) } }}{\sum_{\ell\in\S} H_{n,k}(\ell\mid i)e^{V'_{n-k}(\ell)}} \\
& = \frac{\sum_{j\in\S} \pbr{ H_{n,k}(j\mid i)  e^{V'_{n-k}(j)}} \cdot \pbr{ \frac{e^{W'_{n-k}(j)}}{e^{V'_{n-k}}(j)}} }{ \sum_{\ell\in\S} H_{n,k}\pbr{\ell\mid i} e^{V'_{n-k}(\ell)}}.
}

Define a probability measure $q$ as follows:
\alns{ 
q(j\mid i) := \frac{ H_{n,k} (j\mid i) e^{V'_{n-k}}(j)}{ \sum_{\ell\in\S} H_{n,k} \pbr{\ell\mid i} e^{V'_{n-k}}(\ell)}.
}

Notice that $0 < q(j\mid i) < 1$ since $H_{n,k}(j\mid i) > 0$ for all $i,j$ and $0 < \beta < e^{V'_{n-k}(i)} \leq 1$ (from Lemma \ref{lemm_bound}) for all $i\in\S$. Therefore,
\alns{ 
g_n(i) \leq \sum_{j\in\S} q(j\mid i ) \pbr{ \frac{ \pbr{\Qt_{f_{n-k+1}} e^{V'_{n-k}}}(j)}{ e^{V'_{n-k}}(j)}} = \sum_{j\in\S} q(j\mid i) \pbr{ \frac{\T e^{V'_{n-k}}(j)}{e^{V'_{n-k}}(j) }}.
}

Let $\gamma := \min_{i,j} q(j\mid i) > 0$. We have
\alns{ 
g_n(i) \leq \gamma \ell_{n-k } + \pbr{1-\gamma} u_{n-k} \quad \forall  \ i.
}
\aln{ 
\implies u_n \leq \gamma \ell_{n-k} + \pbr{1-\gamma} u_{n-k}.
\label{eq: ordering} }

Since $\ell_{n-k} \leq e^{\Tilde{\Lambda}^*}$, we have
\alns{ 
u_n \leq \gamma e^{\Tilde{\Lambda}^*} + \pbr{1-\gamma} u_{n-k}.
}

Therefore,
\aln{ 
\pbr{ u_n - e^{\Tilde{\Lambda}^*}} \leq \pbr{1-\gamma}\pbr{ u_{n-k} - e^{\Tilde{\Lambda}^*}}. 
\label{eq: x}}

Since $u_n \leq u_{n-1}$ from Lemma \ref{lemm_mono} and $u_n \geq e^{\Tilde{\Lambda}^*}$ from Lemma \ref{lem_VI}, it follows from~\eqref{eq: x} that
\alns{ 
u_n \longrightarrow e^{\Tilde{\Lambda}^*}.
}

From~\eqref{eq: ordering}, we obtain
\alns{ 
e^{\Tilde{\Lambda}^*} \leq \gamma \ell_{n} + \pbr{1-\gamma} u_{n}.
}

Therefore,
\alns{ 
e^{\Tilde{\Lambda}^* } - u_{n} \leq \gamma \pbr{\ell_{n} - u_{n}},
}

which yields
\alns{ 
0 \leq \gamma \pbr{u_{n} - \ell_{n}} \leq \pbr{u_{n} - e^{\Tilde{\Lambda}^*}}.
}

Since $u_{n} \to e^{\Tilde{\Lambda}^*}$, we conclude that $\ell_n \to u_n \implies \ell_n \to e^{\Tilde{\Lambda}^*}$ as desired. \hfill\Halmos
}
\endproof
From Theorem \ref{theorem2} we can equivalently obtain the original optimal risk sensitive $\Lambda^*$ average cost and the corresponding value function associated with it. Note that if $\frac{\pbr{\T e^{V}}(i)}{e^{V(i)}}=\delta>0$, then  the transformation in Equation \eqref{ogcost} provides a $\Lambda$ which is in a $\delta$-scaled neighbourhood of $\Lambda^*$. More details can be found in \cite{cadena}.  

\section{Simulations}\label{simulations}
In this section, we present simulation results to empirically investigate the convergence behavior of risk sensitive modified policy iteration. Through these simulations, we aim to gain a deeper understanding of how modified policy iteration converges in various scenarios and settings. We primarily test for the following metrics:
\begin{enumerate}
    \item The time taken for modified policy iteration to converge across an array of risk sensitivity factors $\alpha$. We plot these times together with the time taken by value iteration and policy iteration algorithms to show the better performance of modified policy iteration across all values of $\alpha$. We also present these plots for varying values of $m_k.$
    \item We study the empirical dependence of total time taken for convergence of the modified policy iteration algorithm as a function of $m_k$. We study this behaviour for various values of $\alpha$ as well. 
    \item We study the improvement in convergence time as a function of state space and action space cardinalities as a function of $\alpha$, for a fixed value of $m$.
\end{enumerate}

\subsection{Setup}
We consider a Markov decision process with $|\S|$ states and $|\A|$ actions. Transition kernels of dimension $|\S|$ $\times$ $|\S|$ are generated; each row in the transition kernel can take one of $|\A|$ possible values corresponding to the actions. A cost matrix of dimension $|\S|$ $\times$ $|\A|$ is also randomly generated with values in $[0,1].$ 

\textbf{Convergence Metric:} Three algorithms, namely value iteration, policy iteration and modified policy iteration are simulated. The convergence criteria for all three algorithms are identical. Specifically, the convergence is achieved when the absolute difference between two successive value function iterates falls below 1e-7.

\subsection{Results}
In Figure \ref{fig:m} we present the time taken for each of the dynamic programming algorithms to converge as a function of the risk sensitivity factor $\alpha.$ Each plot corresponds to a different value of the input parameter $m$, in Algorithm~\ref{Alg: MPI}. Note that $m_k=m$ for all $k\in\N$. 
\begin{figure}[H]
\centering
\begin{subfigure}{.5\textwidth}
  \centering
  \includegraphics[width=0.95\linewidth]{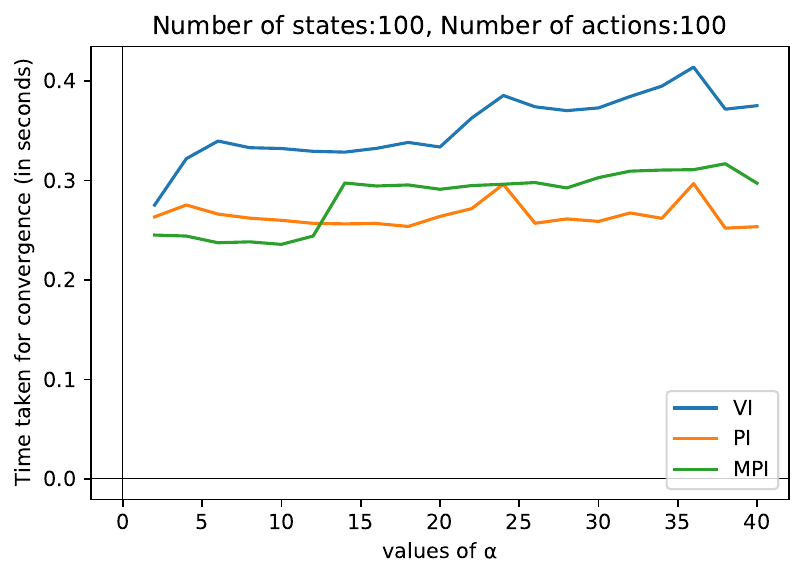}
  \caption[size=10pt]{m=2}
  \label{fig:sub1}
\end{subfigure}%
\begin{subfigure}{.5\textwidth}
  \centering
  \includegraphics[width=0.95\linewidth]{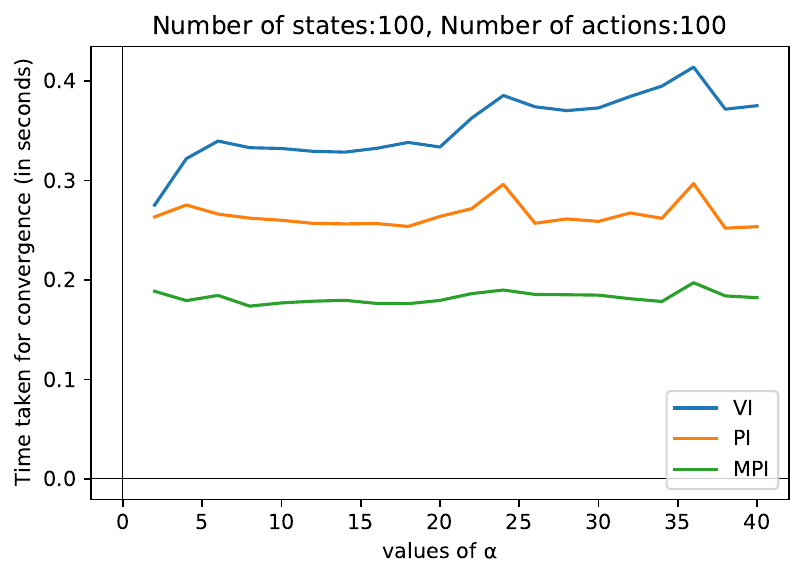}
  \caption[size=10pt]{m=5}
  \label{fig:sub2}
\end{subfigure}
\begin{subfigure}{.5\textwidth}
  \centering
  \includegraphics[width=0.95\linewidth]{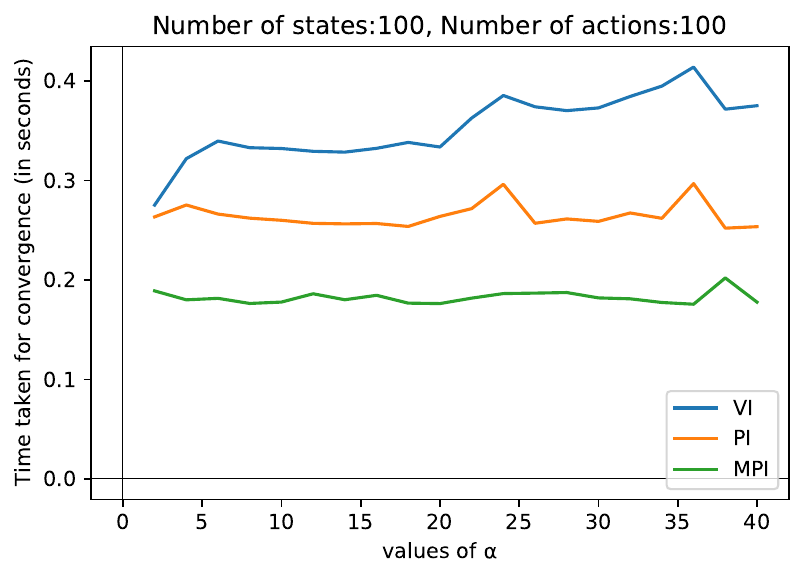}
  \caption[size=10pt]{m=10}
  \label{fig:sub3}
\end{subfigure}%
\begin{subfigure}{.5\textwidth}
  \centering
  \includegraphics[width=0.95\linewidth]{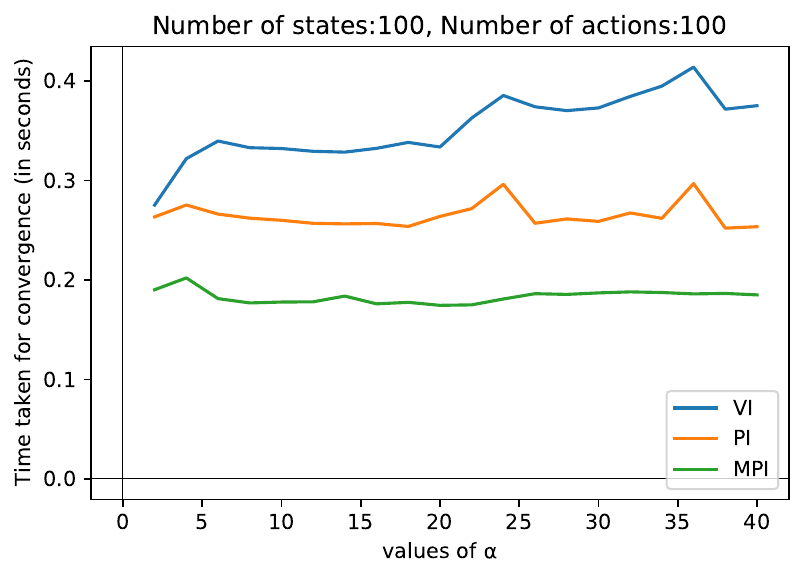}
  \caption[size=10pt]{m=15}
  \label{m=15}
\end{subfigure}
\caption[size=10pt]{Convergence performance of value iteration and policy iteration in comparison with modified policy iteration across various risk sensitivity factors.}
\label{fig:m}
\end{figure}

In Figure \ref{fig:conv} represents the time taken for modified policy iteration to converge, for different values of risk sensitivity parameter $\alpha$, as a function of the algorithm input parameter $m$. 
\begin{figure}[H]
\centering
  \centering
  \includegraphics[width=0.5\linewidth]{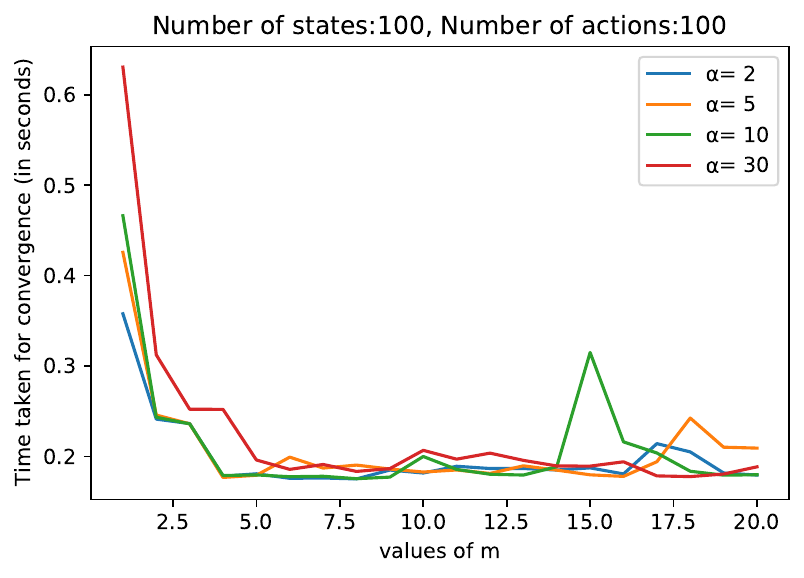}
  \caption{Comparison of time for convergence as a function of $m$}
  \label{fig:conv}
\end{figure}%

In Figure \ref{fig:states} below we present the performance of various dynamic programming algorithms as a function of different risk sensitivity factors $\alpha$ as the cardinality of the state and action spaces increase. 
\begin{figure}[H]
\centering
\begin{subfigure}{.5\textwidth}
  \centering
  \includegraphics[width=0.95\linewidth]{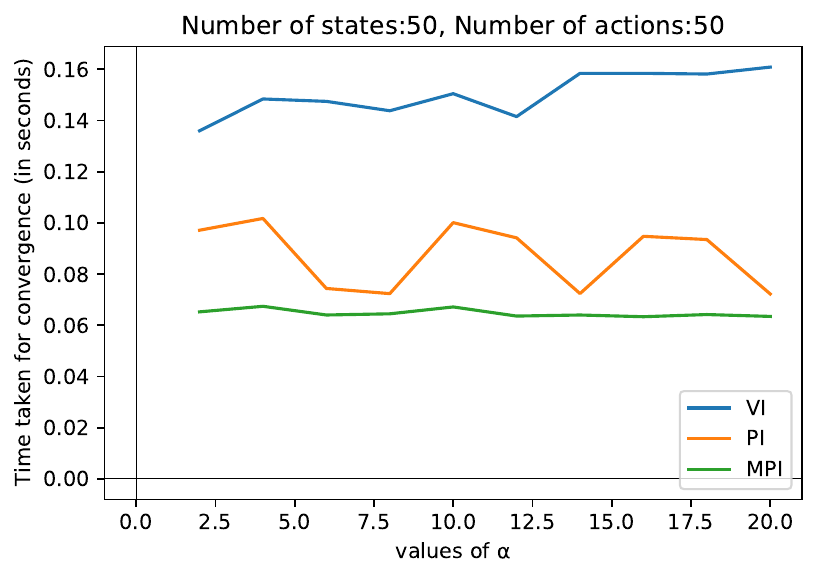}
  \caption{$|\S|=|\A|=50$}
  \label{fig:sub11}
\end{subfigure}%
\begin{subfigure}{.5\textwidth}
  \centering
  \includegraphics[width=0.95\linewidth]{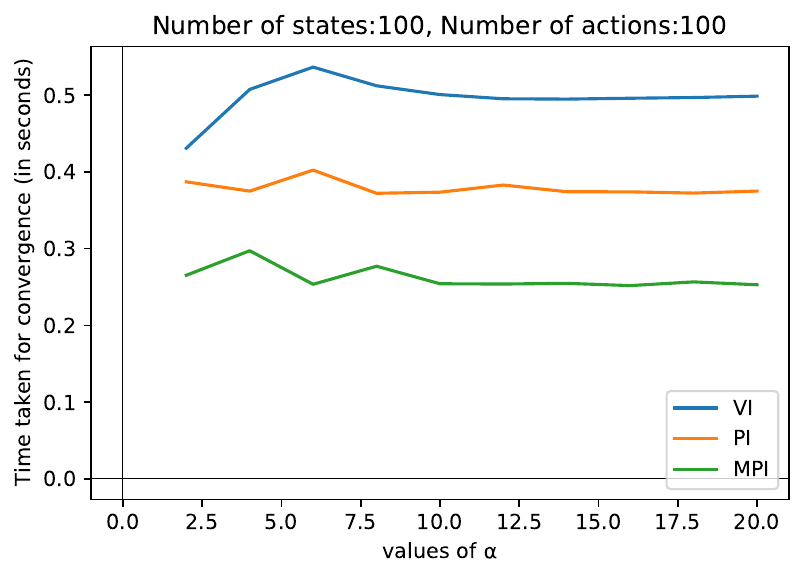}
  \caption{$|\S|=|\A|=100$}
  \label{fig:sub21}
\end{subfigure}
\begin{subfigure}{.5\textwidth}
  \centering
  \includegraphics[width=0.95\linewidth]{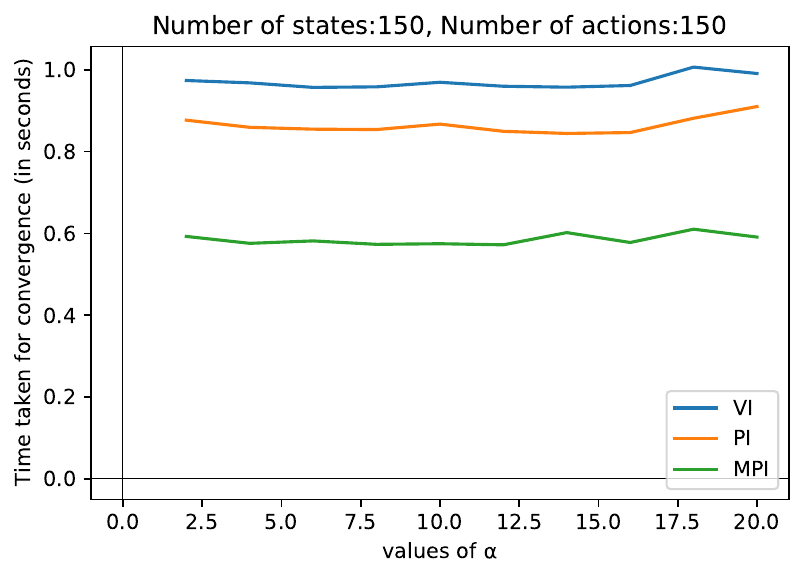}
  \caption{$|\S|=|\A|=150$}
  \label{fig:sub31}
\end{subfigure}%
\begin{subfigure}{.5\textwidth}
  \centering
  \includegraphics[width=0.95\linewidth]{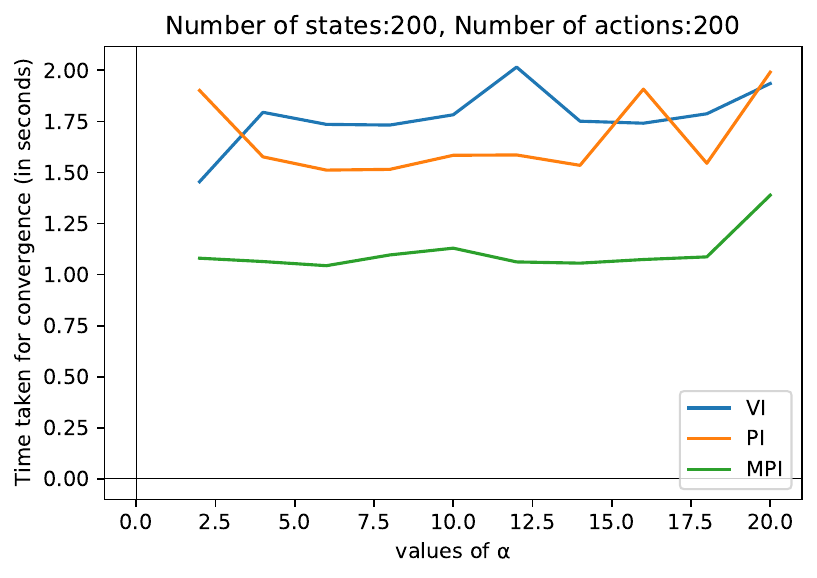}
  \caption{$|\S|=|\A|=200$}
  \label{sub41}
\end{subfigure}
\caption{Convergence performance of value iteration and policy iteration in comparison with modified policy iteration across different state and action space cardinalities.}
\label{fig:states}
\end{figure}

\subsection{Discussion}

Figure \ref{fig:m} shows that for very small values of $m,$ policy iteration outperforms modified policy iteration but once you increase $m,$ modified policy iteration quickly starts outperforming policy iteration. Value iteration seems to be consistently slower than the other two algorithms. From Figure~\ref{fig:conv}, one can see that the running time of modified policy iteration decreases with increasing $m$ but becomes nearly a constant once $m$ reaches a relatively small threshold value. We believe this is due to increased accuracy in policy evaluation being offset by the increased amount of time it takes to evaluate each policy. Figure~\ref{fig:states} shows that the difference between modified policy iteration and the other two algorithms becomes more pronounced as the sizes of the state and action spaces increase.

\section{Conclusion}\label{conclusion}

We presented a modified policy iteration algorithm which can reduce the computational burden of standard policy iteration for risk-sensitive MDPs. We present the proof of convergence for this algorithm, and empirically validate its utility. As in prior work for discounted-cost problems, it would be interesting to investigate if our results can further be used to provide performance guarantees for RL algorithms for risk-sensitive MDPs. 

\section*{Acknowledgments}
This research was supported in part by AFOSR Grant FA9550-24-1-0002, ONR Grant N00014-19-1-2566, and NSF Grants CNS 23-12714, CNS 21-06801, CCF 19-34986, and CCF 22-07547.

\bibliographystyle{informs2014}
\bibliography{myrefs}

\end{document}